\def\year{2020}\relax
\documentclass[letterpaper]{article} 
\usepackage{aaai20}  
\usepackage{times}  
\usepackage{helvet} 
\usepackage{courier}  
\usepackage[hyphens]{url}  
\usepackage{graphicx} 
\usepackage{graphicx}  
\usepackage{epsfig}
\usepackage{amsmath}
\usepackage{amsthm}
\usepackage{amssymb}
\usepackage{multirow}

\usepackage{subfigure}
\usepackage{bm}

\DeclareMathOperator*{\argmin}{argmin}
\usepackage[pagebackref=true,breaklinks=true,letterpaper=true,colorlinks,bookmarks=false]{hyperref}
\newtheorem{proposition}{Proposition}

\usepackage[ruled,vlined]{algorithm2e}

\title{Visualizing Deep Networks by Optimizing with Integrated Gradients}
\author{Zhongang Qi\textsuperscript{\rm 1,2}, Saeed Khorram\textsuperscript{\rm 1}, Li Fuxin\textsuperscript{\rm 1}\\
\textsuperscript{\rm 1}School of Electrical Engineering and Computer Science, Oregon State University\\
\textsuperscript{\rm 2}Applied Research Center, PCG, Tencent\\
{\tt\small zhongangqi@tencent.com, \{khorrams,lif\}@oregonstate.edu}}

\begin{document}

\maketitle

\begin{abstract}
Understanding and interpreting the decisions made by deep learning models is valuable in many domains. In computer vision, computing heatmaps from a deep network is a popular approach for visualizing and understanding deep networks. However, heatmaps that do not correlate with the network may mislead human, hence the performance of heatmaps in providing a faithful explanation to the underlying deep network is crucial. In this paper, we propose  I-GOS, which optimizes for a heatmap so that the classification scores on the masked image would maximally decrease. The main novelty of the approach is to compute descent directions based on the integrated gradients instead of the normal gradient, which avoids local optima and speeds up convergence. 
Compared with previous approaches, our method can flexibly compute heatmaps at any resolution for different user needs. 
Extensive experiments on several benchmark datasets show that the heatmaps produced by our approach are more correlated with the decision of the underlying deep network, in comparison with other state-of-the-art approaches.
\end{abstract}

\section{Introduction}
In recent years, there has been significant focus on explaining deep networks~ \cite{ribeiro2016should,ScottNIPS2017,ElenbergNIPS17,netdissect2017,Zhou_2018_ECCV,ZhangICNN17,DavidSelf18}. Explainability is important for humans to trust the deep learning model, especially in crucial decision-making scenarios.
In the computer vision domain, one of the most important explanation techniques is the heatmap approach \cite{MatDeconv,SimonyanVZ13,Gradcam17,zhang16excitationBP}, which focuses on generating heatmaps that highlight parts of the input image that are most important to the decision of the deep networks on a particular classification target.
The heatmaps can be used to diagnose deep models to understand whether the models utilize the right contents to make the classification. This diagnosis is important when deep networks malfunction in high-stake cases, e.g. autonomous driving. In medical imaging and other image domains that humans currently lack understanding, heatmaps can also be used to help humans gain further insights on which part of the images are important.

\begin{figure*}[t]
\begin{center}
\includegraphics[width=0.95\linewidth]{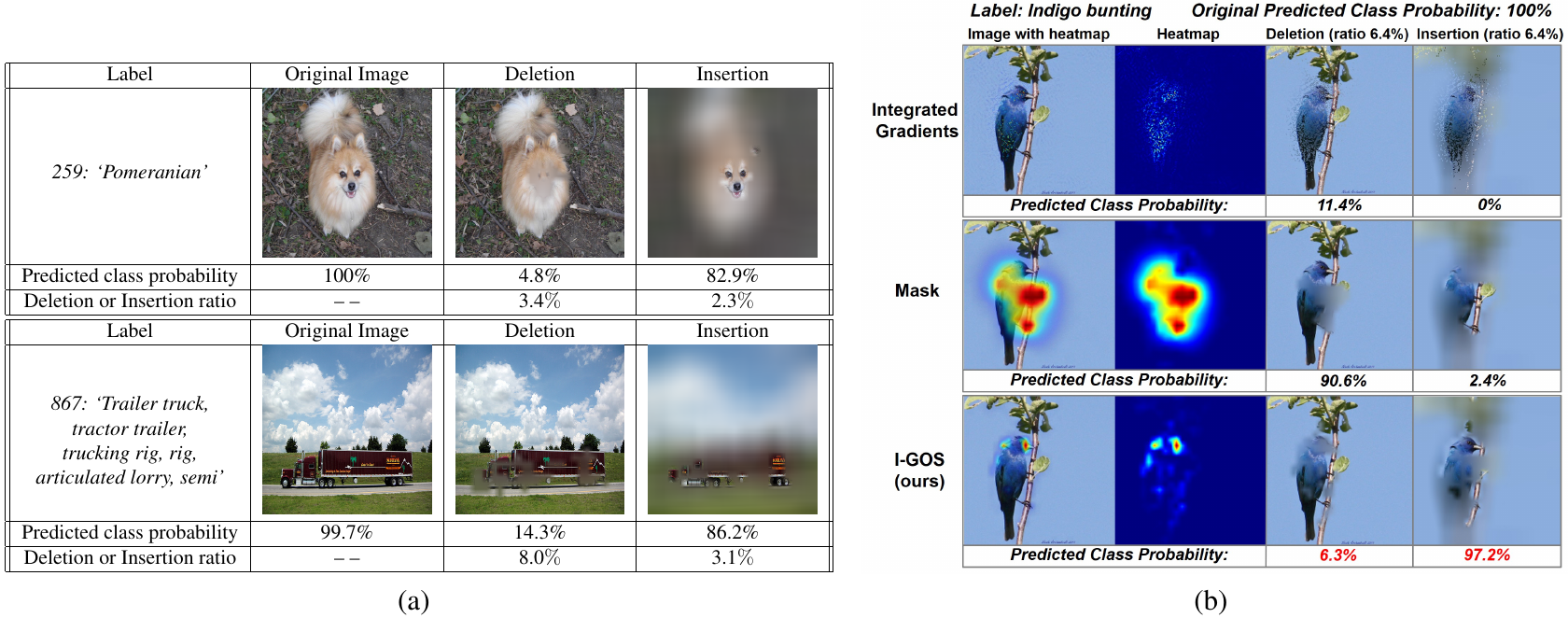}
\end{center}
\vskip -0.15in
\caption{(a) Examples generated by I-GOS in the deletion and insertion tasks using VGG19 as the baseline model. Heatmaps can be verified by testing the CNN on (multiple) deletion images (column 3), which blur the most highlighted areas on the heatmap, and (multiple) insertion images (column 4), which blur areas not highlighted on the heatmap. 
(b) The first two rows show that Integrated Gradients, Mask 
may fail on these evaluations. In the third row, I-GOS performs significiantly better since the CNN no longer classifies the  image to the same category with a small deleted area (column 3), and classifies the image correctly with only few pixels revealed (column 4), showing the correlation between the I-GOS heatmap and CNN decision making. For all approaches the same amount of pixels (6.4\% in this figure) were blurred/revealed. (Best viewed in color)} 
\label{fig:intr}
\end{figure*}

Some heatmap approaches achieve good visual qualities for human understanding, such as several one-step backpropagation-based visualizations including Guided Backpropagation (GBP) \cite{JTguided2015} and the deconvolutional network (DeconvNet) \cite{MatDeconv}. 
These approaches utilize the gradient or variants of the gradient and backpropagate them back to the input image, in order to decide which pixels are more relevant to the change of the deep network prediction.
However, whether they are actually correlated to the decision-making of the network is not that clear~\cite{2018Nie}.
\cite{2018Nie} proves that GBP and DeconvNet are essentially doing (partial) image recovery, and thus generate more human-interpretable visualizations that highlight object boundaries, which do not necessarily represent what the model has truly learned.

An issue with these one-step approaches is that they only reflect infinitesimal changes of the prediction of a deep network. In the highly nonlinear function estimated by the deep network, such infinitesimal changes are not necessarily reflective of changes large enough to alter the decision of the model.  \cite{2018RISE} proposed evaluation metrics based on masking the image with heatmaps and verifying whether the masking will indeed change deep network predictions. Ideally, if the highlighted regions for a category are removed from the image, the deep network should no longer predict that category. This is measured by the \textit{deletion} metric. On the other hand, the network should predict a category only using the regions highlighted by the heatmap, which is measured by the \textit{insertion} metric (Fig.~\ref{fig:intr}).

If these are the goals of a heatmap, a natural idea would be to directly optimize them. 
The mask approach proposed in \cite{ClassicMask} generates heatmaps by solving an optimization problem, which aims to find the smallest and smoothest area that maximally decrease the output of a neural network, directly optimizing the \textit{deletion} metric. 
It can generate very good heatmaps, but usually takes a long time to converge, and sometimes the optimization can be stuck in a bad local optimum due to the strong nonconvexity of the solution space  (Fig.~\ref{fig:intr}). 

In this paper, we propose a novel visualization approach I-GOS (Integrated-Gradients Optimized Saliency) which utilizes an idea called integrated gradients to improve the mask optimization approach in~\cite{ClassicMask}. The integrated gradients approach explicitly find a \textit{baseline} image with very low prediction confidence -- a completely grey or highly blurred image -- then compute a straight line between the original image and this baseline. The gradients on images on this line are then integrated
\cite{IntegratedGradient}. 
The idea is that the direction provided by the integrated gradients may lead better towards the global optimum than the normal gradient which may tend to lead to local optima. However, the original integrated gradient~\cite{IntegratedGradient} paper is a one-step visualization approach and generate diffuse heatmaps difficult for human to understand (Fig.~\ref{fig:intr}). In this paper, we replace the gradient in mask optimization with the integrated gradients. Due to the high cost of computing the integrated gradients, we employ a line-search based gradient-projection method to maximally utilize each computation of the integrated gradients. We also utilize some empirical perturbation strategies to avoid the creation of adversarial masks. In the end, 
our approach generates better heatmaps (Fig.~\ref{fig:intr}) and utilizes less computational time than the original mask optimization, as line search is more efficient in finding appropriate step sizes, allowing significantly less iterations to be used.
We highlight our contributions as follows:
\begin{itemize}
\item[(1)]{We developed a novel heatmap visualization approach I-GOS, which optimizes a mask using the integrated gradients as  descent steps.} 
\item[(2)]{Through regularization and perturbation we better avoided generating adversarial masks at higher resolutions, enabling more detailed heatmaps that are more correlated with the decision-making of the model.}
\item[(3)]{Extensive evaluations show that the proposed approach performs better than the state-of-the-art approaches, especially in the \textit{insertion} and \textit{deletion} metrics.}
\end{itemize}

\section{Related Work}
There are several different types of the visualization techniques for generating heatmaps for a deep network. We classify them into one-step backpropagation-based approaches \cite{MatDeconv,SimonyanVZ13,JTguided2015,InputGradient,IntegratedGradient,LRP15,zhang16excitationBP,Gradcam17}, and perturbation-based approaches, e.g., \cite{Occlude15,Dabkowski2017,ClassicMask,2018RISE}.

The basic idea of one-step backpropagation-based visualizations is to backpropagate the output of a deep neural network back to the input space using the gradient or its variants. 
DeconvNet \cite{MatDeconv}, Saliency Maps \cite{SimonyanVZ13}, and GBP \cite{JTguided2015} are similar approaches, with the difference among them in the way they deal with the ReLU layer. 
LRP \cite{LRP15} and DeepLIFT \cite{InputGradient} compute the contributions of each input feature to the prediction.
Excitation BP \cite{zhang16excitationBP} passes along top-down signals downwards in the network hierarchy via a probabilistic Winner-Take-All process. GradCAM \cite{Gradcam17} uses the gradients of a target concept, flowing only into the final convolutional layer to produce a coarse localization map. \cite{unifyBP} analyzes various backpropagation-based methods, and provides a unified view to explore the connections among them. 

Perturbation-based methods first perturb parts of the input, and then run a forward pass to see which ones are most important to preserve the final decision. The earliest approach \cite{Occlude15} utilized a grey patch to occlude part of the image. This approach is direct but very slow, usually taking hours for a single image \cite{unifyBP}. An improvement is to introduce a mask, and solve for the optimal mask as an optimization problem \cite{Dabkowski2017,ClassicMask}. \cite{Dabkowski2017} develop a trainable masking model that can produce the masks in a single forward pass. However, it is difficult to train a mask model, and different models need to be trained for different networks. \cite{ClassicMask} directly solves the optimization, and find the mask iteratively. Instead of only occluding one patch of the image, RISE \cite{2018RISE} generates thousands of randomized input masks simultaneously, and averages them by their output scores. However, it consumes significant time and GPU memory.


Another seemingly related but different domain is the saliency map from human fixation \cite{eyefix}. Fixation Prediction \cite{Kruthiventi_2016_CVPR,Kummerer_2017_ICCV} aims to identify the fixation points that human viewers would focus on at first glance of a given image, usually by training a network to predict those fixation points. This is different from deep explanation because deep models may use completely different mechanisms to classify from humans, hence human fixations should not be used to train or evaluate heatmap models.

\section{Model Formulation}

\subsection{Gradient and Mask Optimization}
Gradient and its variants are often utilized in visualization tools to demonstrate the importance of each dimension of the input. Its motivation comes from the linearization of the model. Suppose a black-box deep network $f$ predicts a score $f_c(I)$  on class $c$ (usually the logits of a class before the softmax layer) from an image $I$. Assume $f$ is smooth at the current image $I_0$, then a local approximation can be obtained using the first-order Taylor expansion:
{\small
\begin{align}
f_c(I) \approx f_c(I_0) + \langle \nabla f_c(I_0), I-I_0 \rangle,
\end{align}
}The gradient $\nabla f_c(I_0)$ is indicative of the local change that can be made to $f_c(I_0)$ if a small perturbation is added to it, and hence can be visualized as an indication of salient image regions to provide a local explanation for image $I_0$ \cite{SimonyanVZ13}.
In \cite{InputGradient}, the heatmap is computed by multiplying the gradient feature-wise with the input itself, i.e., $\nabla f_c(I_0) \odot I_0$, to improve the sharpness of heatmaps.


However, gradient only illustrates the infinitesimal change of the function $f_c(I)$ at $I_0$, 
which is not necessarily indicative of the salient regions that lead to a significant change on $f_c(I)$, especially when the function is highly nonlinear.
What we would expect is that the heatmaps indicate the areas that would really change the classification result significantly.
In \cite{ClassicMask}, a perturbation based approach is proposed which introduces a mask $M$ as the heatmap to perturb the input $I_0$.
$M$ is optimized by solving the following objective function:
{\small
\begin{align}
&\argmin_M~ F_c(I_0, M) = f_c\big(\Phi(I_0, M)\big) + g(M), \notag\\
&\text{where~~} 
 g(M) = \lambda_1 ||{\bf 1}-M||_1 +\lambda_2 \text{TV}(M), \label{eq:classicMask}\\
&~~~~~~~~\Phi(I_0, M) = I_0 \odot  M + \tilde{I}_0 \odot ({\bf 1}-M), ~~~~~{\bf 0} \leq M \leq {\bf 1}, \notag
\end{align}
}In (\ref{eq:classicMask}), $M$ is a matrix which has the same shape as the input image $I_0$ and whose elements are all in $[0,1]$; $\tilde{I}_0$ is a baseline image with the same shape as $I_0$, which should have a low score on the class $c$, 
$f_c\big(\tilde{I_0}\big) \approx \min_I f_c(I)$, 
and in practice either a constant image, random noise, or a highly blurred version of $I_0$. This optimization seeks to find a deletion mask that significantly decreases the output score
$f_c\big(\Phi(I_0, M)\big)$, i.e., $f_c\big(I_0 \odot M + \tilde{I}_0 \odot ({\bf 1}-M)\big) \ll f_c(I_0)$ under the regularization of $g(M)$. $g(M)$ contains two regularization terms, with the first term on the magnitude of $M$, and the second term a total-variation (TV) norm \cite{ClassicMask} to make $M$ more piecewise-smooth.



Although this approach of optimizing a mask performs significantly better than the gradient method, there exist inevitable drawbacks when using a traditional first-order algorithm to solve the optimization. 
First, it is slow, usually taking hundreds of iterations to obtain the heatmap for each image. 
Second, since the model $f_c$ is highly nonlinear in most cases, optimizing (\ref{eq:classicMask}) may only achieve a local optimum, with no guarantee that it indicates the right direction for a significant change related to the output class. 
Fig. \ref{fig:intr} and Fig. \ref{fig:compare1} show some heatmaps generated by the mask approach.

\subsection{Integrated Gradients}

Note that the problem of finding the mask is not a conventional non-convex optimization problem. For $F_c(I_0,M) = f_c(I_0,M) +g(M)$, we (approximately) know the global minimum (or, at least a reasonably small value) of $f_c(I_0, M)$ in a baseline image $\tilde{I}_0$, which corresponds to $M = \mathbf{0}$.
The integrated gradients \cite{IntegratedGradient} consider the straight-line path from the baseline $\tilde{I}_0$ to the input $I_0$. Instead of evaluating the gradient at the provided input $I_0$ only, the integrated gradients would be obtained by accumulating all the gradients along the path:
{\small
\begin{align}
IG_i(I_0) =    (I_0^i-\tilde{I}_0^i) \cdot \int_{\alpha=0}^{1} \frac{\partial f_c\big(\tilde{I}_0 + \alpha(I_0-\tilde{I}_0)\big)}{\partial I_0^i} d\alpha ,
\label{eq:IG}
\end{align}
}where $IG(I_0) = \nabla_{I_0}^{IG}f_c(I_0)$ is the integrated gradients of $f_c$ at $I_0$; $i$ represents the $i$-th pixel. 
%


In practice, the integral in (\ref{eq:IG}) is approximated via a summation.
We sum the gradients at points occurring at sufficiently small intervals along the straight-line path from the input $M$ to a baseline $\tilde{M}=\mathbf{0}$:
{\small
\begin{align}
\nabla^{IG} f_c(M) =\frac{1}{S} \sum_{s=1}^{S} \frac{ \partial f_c\left(\Phi\big(I_0, \frac{s}{S}M\big)\right)}{\partial M},
\label{eq:bIG}
\end{align}
}where $S$ is a constant, usually $20$. However, \cite{IntegratedGradient} only proposed to use integrated gradients as a one-step visualization method, 
and the heatmaps generated by the integrated gradients are still diffuse. 
%
Fig. \ref{fig:intr} and Fig. \ref{fig:compare1} show some heatmaps generated by the integrated gradients approach where a grey zero image is utilized as the baseline.
We can see that the integrated gradient contains many false positives in the area wherever the pixels have a large value of $I_0^i-\tilde{I}_0^i$ (either the white or the black pixels).



\subsection{Integrated Gradients Optimized Heatmaps}
We believe the above two approaches can be combined for a better heatmap approach. The integrated gradient naturally provides a better direction than the gradient in that it points more directly to the global optimum of a part of the objective function. 
One can view the convex constraint function $g(M)$ as equivalent to the Lagrangian of a constrained optimization approach with constraints $\|{\bf 1}-M\|_1 \leq B_1$ and $TV(M) \leq B_2$, $B_1$ and $B_2$ being positive constants, hence consider the optimization problem (\ref{eq:classicMask}) to be a constrained minimization problem on $f_c(\Phi(I_0, M))$. In this case, we know the unconstrained solution in $M = {\bf 0}$ is outside the constraint region. We speculate that an optimization algorithm may be better than gradient descent if it directly attempts to move to the unconstrained global optimum.


\begin{figure}[tb]
\centering
\includegraphics[width=0.78\columnwidth]{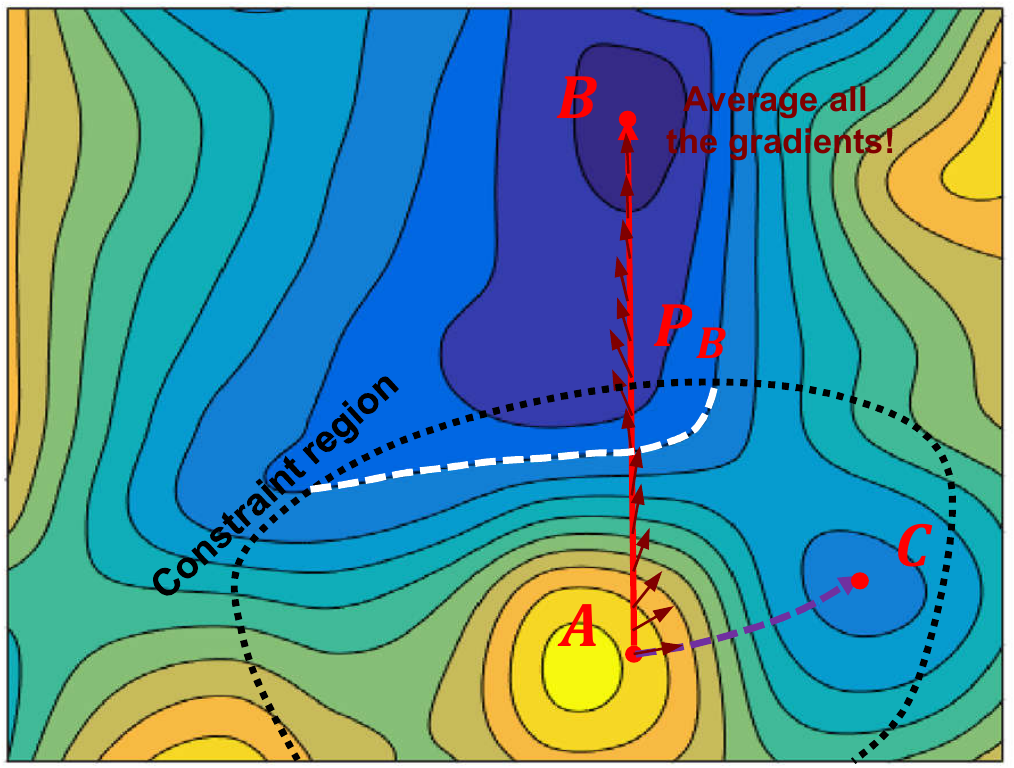}
\vskip -0.05in
\caption{(Best viewed in color) Suppose we are optimizing in a region with a starting point $A$, a local optimum $C$, and a baseline $B$ which is the unconstrained global optimum; the area within the black dashed line is the constraint region which is decided by the constraint terms $g(I,M)$ and the bound constraints $\mathbf{0} \leq M \leq \mathbf{1}$, we may find a better solution by always moving towards $B$ rather than following the gradient and end up at $C$.}
\label{fig:1}
\end{figure}

To illustrate this, Fig.
\ref{fig:1} shows a 2D optimization with a starting point $A$, a local optimum $C$, and a baseline $B$.
The area within the black dashed line is the constraint region which is decided by the constraint function $g(M)$ and the boundary of $M$.
A first-order algorithm will follow the gradient descent direction (the purple line) to the local optimum $C$; 
while the integrated gradients computed along the path $P_B$ from $A$ to the baseline $B$ may enable the optimization to reach an area better than $C$ within the constraint region. We can see that the integrated gradients with an appropriate baseline have a global view of the space and may generate a better descent direction.
In practice, the baseline does not need to be the global optimum. A good baseline near the global optimum could still improve over the local optimum achieved by gradient descent.

Hence, we utilize the integrated gradients to substitute the gradient of the partial objective $f_c(M)$ in optimization (\ref{eq:classicMask}), and introduce a new visualization method called Integrated-Gradient Optimized Saliency (I-GOS). For the regularization terms $g(M)$ in optimization (\ref{eq:classicMask}), we still compute the partial (sub)gradient with respect to $M$:

{\small
\begin{align}
\nabla g(M) = \lambda_1 \cdot \frac{\partial||{\bf 1} -M||_1}{\partial M} + \lambda_2 \cdot \frac{\partial \text{TV}(M)}{\partial M},
\end{align}
}

The total (sub)gradient of the optimization for $M$ at each step is the combination of the integrated gradients for the $f_c(M)$ and the gradients of the regularization terms $g(M)$:
{\small
\begin{align}
{TG}(M) = \nabla^{IG} f_c(M) + \nabla g(M),
\label{eq:total}
\end{align}
}Note that this is no longer a conventional optimization problem, since it contains $2$ different types of gradients.
The integrated gradients are utilized to indicate a direction for the partial objective $f_c(M)$; the gradients of the $g(M)$ are used to regularize this direction and prevent it to be diffuse. 

\begin{figure*}[t]
\centering
\includegraphics[width=1.8\columnwidth]{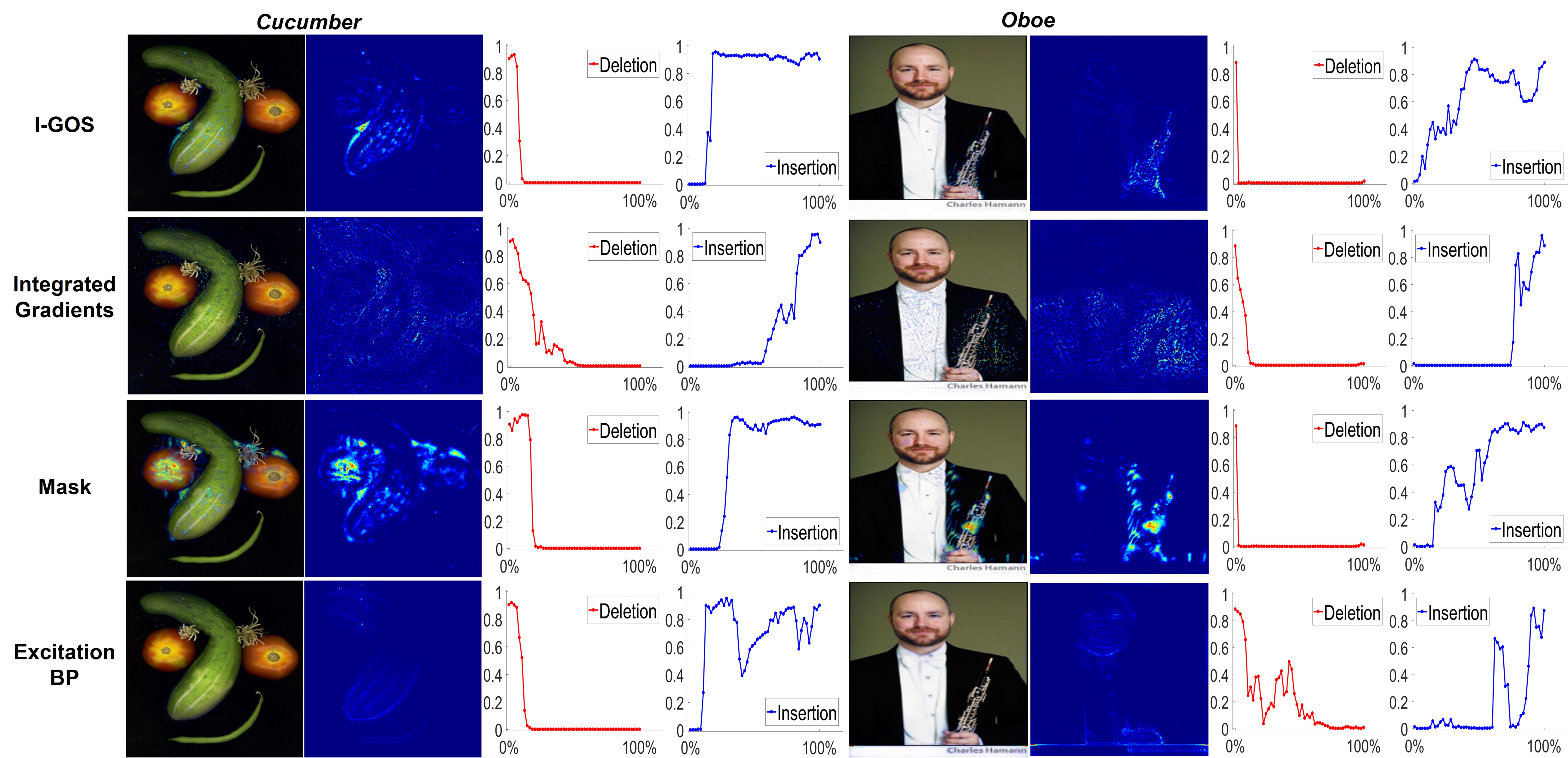}
\caption{Different heatmap approaches at $224\times 224$ resolution. The red plot illustrates how the CNN predicted probability drops with more areas masked, and the blue plot illustrates how the prediction increases with more areas revealed. 
The x axis for the red/blue plot represents the percentage of pixels masked/revealed;
and the y axis represents the predicted class probability.
One can see with I-GOS the red curve drops earlier and the blue plot increase earlier, leading to more area under the insertion curve (insertion metric) and less area under the deletion curve (deletion metric). (Best viewed in color) } 
\label{fig:compare1}
\end{figure*}

\subsection{Computing the step size}
\label{sec:step}
Since the time complexity of computing $\nabla^{IG} f_c(M)$ is high, we utilize a backtracking line search method and revise the Armijo condition \cite{numerical2000} to help us compute the appropriate step size for the total gradient $TG(M)$ in formula (\ref{eq:total}). 
The Armijo condition tries to find a step size such that:
{\small
\begin{align}
f(M_k+\alpha_k \cdot d_k) - f(M_k) \leq \alpha_k \cdot \beta \cdot \nabla f(M_k)^Td_k,
\label{eq:Armijo}
\end{align}
}where $d_k$ is the descent direction; $\alpha_k$ is the step size; $\beta$ is a parameter in $(0,1)$; $\nabla f(M_k)$ is the gradient of $f$ at point $M_k$.

The descent direction $d_k$ for our algorithm is set to be the inverse direction of the total gradient $TG(M_k)$.
However, since  $TG(M_k)$ contains the integrated gradients, it is uncertain whether $\nabla F_c(M_k)^Td_k = -\nabla F_c(M_k)^T TG(M_k)$ is negative or not. Hence, we replace $\nabla F_c(M_k)$ with $TG(M_k)$ and obtain a revised Armijo condition as follows:
{\small
\begin{align}
F_c\bigg(M_k-\alpha_k \cdot TG(M_k)\bigg) - F_c(M_k) \leq  \notag\\
-\alpha_k \cdot \beta \cdot TG(M_k)^{T}TG(M_k),
\label{eq:Armijo2}
\end{align}
}
The detailed backtracking line search works as follows:
{\begin{itemize}
\item[(1)]{Initialization: set the values of the parameter $\beta$, a decay $\eta$, a upper bound $\alpha_{u}$ and a lower bound $\alpha_{l}$ for the step size; let $j=0$, and $\alpha^0 = \alpha_{u}$;}
\item[(2)]{Iteration: if $\alpha^j$ satisfies condition (\ref{eq:Armijo2}), or $\alpha_j \leq \alpha_{l}$, end iteration; else, let $\alpha^{j+1}=\alpha^j \eta$, $j=j+1$, test condition (\ref{eq:Armijo2}) again with  $P_{[0,1]}(M_k-\alpha_k \cdot TG(M_k))$, where $P_{[0,1]}(M)$ clips the mask values to the closed interval $[0,1]$;}
\item[(3)]{Output: if $\alpha^j \leq \alpha_{l}$, the step size $\alpha_k$ for $TG(M_k)$ equals to the lower bound $\alpha_l$; else, $\alpha_k=\alpha^j$}
\end{itemize}}
A projection step is needed in the iteration because the mask $M_k$ is bounded by the closed interval $[0,1]$.
Since we have an integrated gradient in $TG(M)$, a large upper bound $\alpha_u$ and a small $\beta$ are needed in order to obtain a large step that satisfies condition (\ref{eq:Armijo2}), similar to satisfying the Goldstein conditions for convergence in conventional Armijo-Goldstein line search. 

Note that we cannot prove the convergence properties of the algorithm in non-convex optimization. However, the integrated gradient reduces to a scaling on the conventional gradient in a quadratic function (see supplementary material). 
In practice, it converges much faster than the original mask approach in ~\cite{ClassicMask} and we have never observed it diverging, although in some cases we do note that even with this approach the optimization stops at a local optimum. With the line search, we usually only run the iteration for $10-20$ steps. Intuitively, the irrelevant parts of the integrated gradients are controlled gradually by the regularization function $g(M)$ and only the parts that truly correlate with output scores would remain in the final heatmap.
 \newsavebox{\insertionvgg}
 \begin{lrbox}{\insertionvgg}
\begin{tabular}{||l|c|c|c|c|c|c|c|c||}
\hline
  \multirow{2}{*}{}   & \multicolumn{2}{c|}{224$\times$224}      & \multicolumn{2}{c|}{112$\times$112}      & \multicolumn{2}{c|}{28$\times$28}        & \multicolumn{2}{c||}{14$\times$ 14}        \\ \cline{2-9}
                    & Deletion        & Insertion       & Deletion        & Insertion       & Deletion        & Insertion       & Deletion        & Insertion       \\ \hline\hline
Excitation BP \cite{zhang16excitationBP}      & 0.2037          & 0.4728          & 0.2053          & 0.4966          & 0.2202          & 0.5256          & 0.2328          & 0.5452          \\ \hline
Mask \cite{ClassicMask}       & 0.0482          & 0.4158          & 0.0728          & 0.4377          & 0.1056          & 0.5335          & 0.1753          & 0.5647          \\ \hline
GradCam \cite{Gradcam17}             & -- --              &  -- --               &  -- --               &  -- --               &  -- --               &  -- --               & 0.1527          & 0.5938          \\ \hline
RISE  \cite{2018RISE}              & 0.1082          & 0.5139          &  -- --              &  -- --               &  -- --               &  -- --               &  -- --               &  -- --               \\ \hline
Integrated Gradients \cite{IntegratedGradient} & 0.0663          & 0.2551          &  -- --               &  -- --               &  -- --               &  -- --               &  -- --               &  -- --               \\ \hline\hline
I-GOS (ours)      & \textbf{0.0336} & \textbf{0.5246} & \textbf{0.0609} & \textbf{0.5153} & \textbf{0.0899} & \textbf{0.5701} & \textbf{0.1213} & \textbf{0.6387} \\ \hline
\end{tabular}
 \end{lrbox}

 \begin{table*}  
\caption{{ Evaluation in terms of deletion (lower is better) and insertion (higher is better)
scores on the ImageNet dataset using the VGG19 model. GradCam can only generate $14\times 14$ heatmaps for VGG19; RISE and Integrated Gradients can only generate $224\times 224$ heatmaps}}
 \centering   
       \scalebox{0.80}{\usebox{\insertionvgg}}
\label{tab:insertionvgg}
 \end{table*}

\begin{figure*}[t]
\centering
\includegraphics[width=1.8\columnwidth]{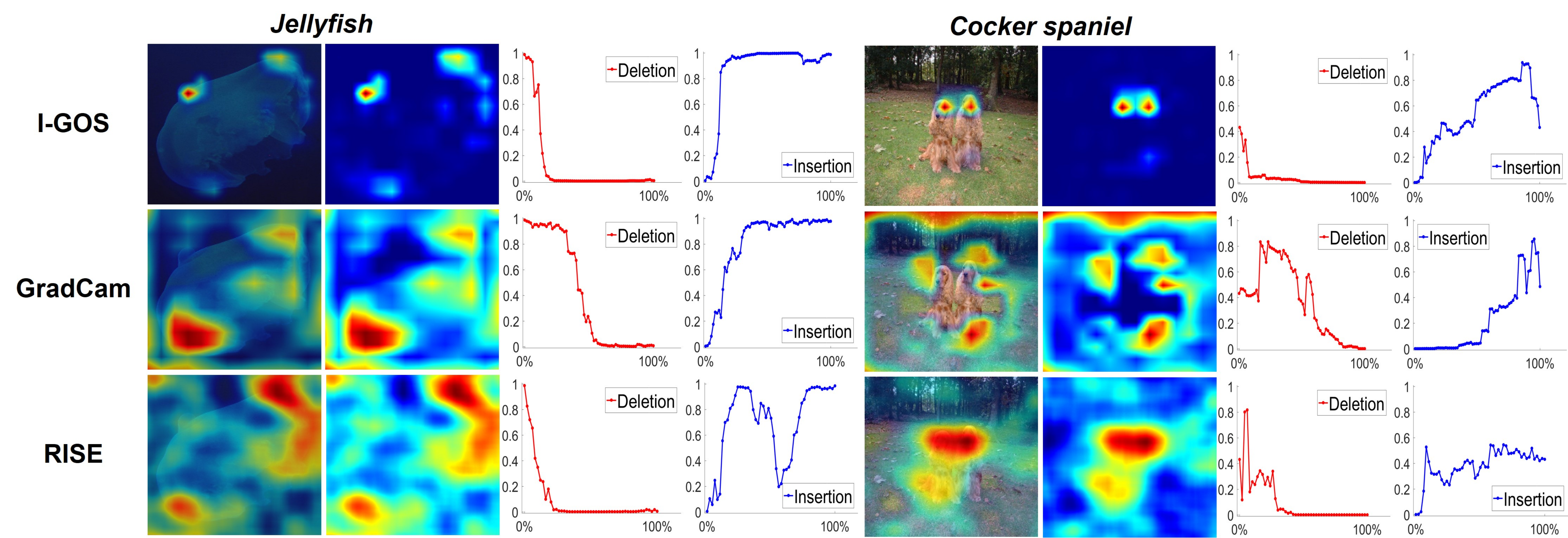}
\caption{Comparisons between GradCam, RISE, and I-GOS, see Fig. 3 caption for explanations of the meaning of the curves.}
\label{fig:GradCamRISE}
\end{figure*}

\subsection{Avoiding adversarial examples}
\label{sec:ad}
Since the mask optimization (\ref{eq:classicMask}) is similar to the adversarial optimization \cite{szegedy2013,goodfellow2014} except the TV term, it is concerning whether the solution would merely be an adversarial attack to the original image rather than explaining the relevant information.
An adversarial example is essentially a mask that drives the image off the natural image manifold, hence the approach in ~\cite{ClassicMask} utilize a blurred version of the original image as the baseline to avoid creating strong adversarial gradients off the image manifold. We follow \cite{ClassicMask} and also use a blurred image as the baseline. The total variation constraints also defeats adversarial masks by making the mask piecewise-smooth.  We also added other methods to avoid finding an adversarial perturbation:

\begin{algorithm}[]
\SetAlgoLined
 {\small {\bf Optimization objective}: formula (\ref{eq:upsample})\;
 {\bf Initialization}: set $M_0 =\mathbf{1}$\;
 \While{not converged and within the maximum steps}{
  {\bf Add} different random noise $n_s$ to $I_0$ when computing the integrated gradient: $\nabla^{IG} f_c(M_k) =\frac{1}{S} \sum_{s=1}^{S} \partial f_c\left(\Phi\big(I_0+n_s, \frac{s}{S}\text{up}(M_k)\big)\right)/\partial M_k$ \;
  {\bf Compute} the total (sub)gradient $TG(M_k)$ of the optimization for $M_k$ using formula (\ref{eq:total})\;
  {\bf Compute} the step size $\alpha_k$ using the introduced backtracking line search algorithm \;
  {\bf Update}: $M_{k+1} = M_k - \alpha_k \cdot TG(M_k)$\;
 }}
 \caption{\small I-GOS}
\label{alg:1}
\end{algorithm}

(1) When computing the integrated gradients using formula (\ref{eq:bIG}), we add different random noise $n_s$ to $I_0$ at each point along the straight-line path.

(2) When computing a mask $M$ whose resolution is smaller than that of the input image $I_0$, we upsample it before perturbing the input $I_0$, and rewrite formula (\ref{eq:classicMask}) as:
{\small
\begin{align}
M^* = \argmin~ &f_c\big(\Phi(I_0, \text{up}(M))\big) + \lambda_1 ||{\bf 1}-M||_1  \notag\\
&+\lambda_2 \text{TV}(M), 
\label{eq:upsample}
\end{align}
}where up($M$) upsamples $M$ to the original resolution with bilinear upsampling. The resolution of $M$ is lower, the generated heatmap is smoother.


Whether a mask is adversarial can be evaluated using the \textit{insertion metric}, detailed in the experiments section. 
We summarize an overview of the proposed I-GOS approach in Algorithm \ref{alg:1}.


 \newsavebox{\insertionres}
 \begin{lrbox}{\insertionres}
 \begin{tabular}{||l|@{  }c@{  }|@{  }c@{  }|@{  }c@{  }|@{  }c@{  }|@{  }c@{  }|@{  }c@{  }|@{  }c@{  }|@{  }c@{  }|@{  }c@{  }|@{  }c@{  }||}
\hline
    \multirow{2}{*}{}                 & \multicolumn{2}{c|}{224$\times$224}     & \multicolumn{2}{c|}{112$\times$112}      & \multicolumn{2}{c|}{28$\times$28}        & \multicolumn{2}{c|}{14$\times$14}        & \multicolumn{2}{c||}{7$\times$7}          \\ \cline{2-11}
                    & Deletion       & Insertion       & Deletion        & Insertion       & Deletion        & Insertion       & Deletion        & Insertion       & Deletion        & Insertion       \\ \hline\hline
Mask \cite{ClassicMask}       & 0.0468         & 0.4962          & 0.0746          & 0.5090           & 0.1151          & 0.5559          & 0.1557          & 0.5959          & 0.2259          & 0.6003          \\ \hline
GradCam  \cite{Gradcam17}            &  -- --            & -- --              &  -- --              &  -- --              &  -- --              &  -- --              &  -- --              & -- --              & 0.1675          & 0.6521          \\ \hline
RISE \cite{2018RISE}          & 0.1196         & 0.5637          &   -- --       & -- --         & -- --        &  -- --         & -- --        & -- --        & -- --          &  -- --        \\ \hline
Integrated Gradients \cite{IntegratedGradient} & 0.0907         & 0.2921          &  -- --             &  -- --             &  -- --              &  -- --              &  -- --              &  -- --              &  -- -- &  -- --             \\ \hline
I-GOS (ours)        & \textbf{0.0420} & \textbf{0.5846} & \textbf{0.0704} & \textbf{0.5943} & \textbf{0.1059} & \textbf{0.5986} & \textbf{0.1387} & \textbf{0.6387} & \textbf{0.1607} & \textbf{0.6632} \\ \hline
\end{tabular}
 \end{lrbox}

 \begin{table*}   
 \caption{{\small Comparative evaluation in terms of deletion (lower is better) and insertion (higher is better)
scores on ImageNet using ResNet50 as the base model. GradCam can only generate $7\times 7$ heatmaps for ResNet50; RISE and Integrated Gradients only generate $224\times 224$ heatmaps}}
 \centering   
       \scalebox{0.80}{\usebox{\insertionres}}
 \label{tab:insertionres}
 \end{table*}

\begin{figure*}[tbp]
\centering
\includegraphics[width=1.9\columnwidth]{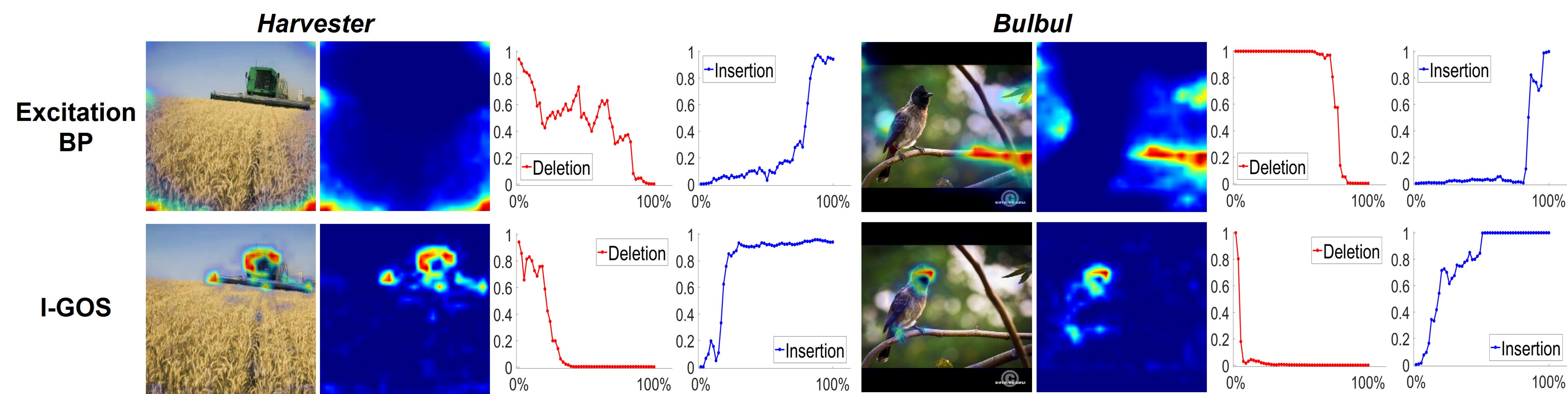}
\vskip -0.1in
\caption{Comparison between Excitation BP and I-GOS at resolution  $28\times 28$. See Fig. 3 for explanations of the figures} 
\label{fig:exbp}
\end{figure*}

\section{Experiments}
\subsection{Evaluation Metrics and Parameter Settings}
Although many recent work focus on explainable machine learning, there is still no consensus about how to measure the explainability of a machine learning model. For the heatmaps, one of the important issues is whether we are explaining the image with human's understanding or the deep model's perspective. A common pitfall is to try to use human's understanding to explain the deep model, e.g. the pointing game \cite{zhang16excitationBP}, which measures the ability of a heatmap to focus on the ground truth object bounding box. However, there are plenty of evidences that deep learning sometimes uses background context for object classification which would invalidate pointing game evaluations. Many heatmap papers show appealing images which look plausible to humans, but \cite{2018Nie} points out they could well be just doing partial image recovery and boundary detection, hence generate human-interpretable results that do not correlate with network prediction. Hence, it is important to utilize objective metrics that have causal effects on the network prediction for the evaluation.


We follow \cite{2018RISE} to adopt {\em deletion} and {\em insertion} as better metrics to evaluate different heatmap approaches. In the {\em deletion} metric, we remove $N$ pixels (dependent on the resolution of the mask)most highlighted by the heatmap  each time from the original image iteratively until no pixel is left, and replace the removed ones with the corresponding pixels from the baseline image.
The deletion score is the area under the curve (AUC) of the classification scores after softmax~\cite{2018RISE} (red curve in Fig. \ref{fig:compare1}-\ref{fig:exbp}).
For the {\em insertion} metric, we replace $N$ most highlighted pixels from the baseline image with the ones from the original image iteratively until no pixel left (blue curve in Fig.\ref{fig:compare1}-\ref{nonoise}). The insertion score is also the AUC of the classification scores for all the replaced images.
In the experiments, we generate heatmaps with different resolutions, including $224\times 224$, $112\times 112$, $28\times 28$, $14\times 14$, and $7\times 7$. 
And we compute the deletion and insertion scores by replacing pixels based on generated heatmaps at the original resolutions before upsampling.


The intuition behind the {\em deletion} metric is that the removal of the pixels most relevant to a class will cause the prediction confidence to drop  sharply. This is similar to the optimization goal in eq. (\ref{eq:classicMask}). Hence, only utilizing the {\em deletion} metric is not satisfactory enough since adversarial attacks can also achieve a quite good performance.
The intuition behind the {\em insertion} metric is that only keeping the most relevant pixels will retain the original score as much as possible. Since adversarial masks usually only optimize the deletion metric, it often use irrelevant parts of the image to drop the prediction score. Thus, if only those parts are revealed to the model, usually the model would not make a confident prediction on the original class, hence a low insertion score. 
Therefore, a good {\em insertion} metric indicates a non-adversarial mask. However, only using the insertion metric would not identify blurry masks (e.g. Fig.~\ref{fig:GradCamRISE}), hence the {\em deletion}-{\em insertion} metrics should be considered jointly.


For the deletion and insertion task, we utilize the pretrained VGG19 \cite{VGG19} and Resnet50 \cite{ResNet} networks from the PyTorch model zoo to test $5,000$ randomly selected images from the validation set of ImageNet \cite{ImageNet}.
In Eq. (\ref{eq:Armijo2}), $\beta= 0.0001$. $\lambda_1$ and $\lambda_2$ in Eq. (\ref{eq:upsample}) were fixed across all experiments under the same heatmap resolution.

We downloaded and ran the code for most baselines, except for \cite{IntegratedGradient} which we implemented. All baselines were tuned to best performances. For RISE, we followed \cite{2018RISE} to generate $4,000$ $7\times 7$ random samples for VGG, and $8,000$ $7\times 7$ random samples for ResNet. For all experiments we used the same pre-/post-processing with the same baseline image $\tilde I_0$. 
\cite{2018RISE} used a less blurred image for insertion and a grey image for deletion.
Since we found the blurriness in \cite{2018RISE} was not always enough to get the CNN to output $0$ confidence, we used a more blurred image for both insertion and deletion, hence the insertion and deletion  scores for RISE are bit different in our paper compared with theirs.

\subsection{Results and Discussions}

{\bf Deletion and Insertion:}
Table \ref{tab:insertionvgg} and \ref{tab:insertionres} show the comparative evaluations of I-GOS with other state-of-the-art approaches in terms of the {\em deletion} and {\em insertion} metrics on the ImageNet dataset using VGG19 and ResNet50 as the baseline model, respectively.
From Table \ref{tab:insertionvgg} and \ref{tab:insertionres} we observe that our proposed approach I-GOS performs better than all baselines in both deletion and insertion scores for heatmaps with all different resolutions. 


Integrated Gradients obtains the worst insertion score among all the approaches, which indicates that it indeed contains lots of pixels uncorrelated with the classification, as in the {\em Cucumber} and {\em Oboe} examples in Fig. \ref{fig:compare1}.
Excitation BP sometimes fires on irrelevant parts of the image as argued in \cite{2018Nie}.
Thus, it performs the worst in the deletion task. GradCAM and RISE also suffer on the deletion score maybe because of the randomness on the masks they generate, which sometimes fixate on random background regions irrelevant to classification.
Fig. \ref{fig:compare1}-\ref{fig:exbp} shows some visual comparisons between our approach and baselines at various resolutions. 
The reason of insertion curve going down and up is that sometimes part of the image that contains features that are indicative of other classes could be inserted, which could increase the activation for other classes, potentially driving down the softmax probability for the current class.

Note that one advantage of our approach compared to the previous best RISE and GradCAM is the flexibility in terms of resolutions. RISE and Integrated Gradients can only generate $224 \times 224$ heatmaps.
GradCam can only generate $14 \times 14$ heatmap on VGG19, and $7\times 7$ heatmap on Resnet50, respectively. Our approach is better than them at their resolutions, but also offers the flexibility to use other resolutions. High resolutions are necessary especially when the image has thin parts (e.g. Fig.~\ref{nonoise}), however may be less visually appealing since the masked pixels may be sparse. Our approach is significantly better than all baselines that can operate on all resolutions. Note that, the insertion metric is higher at lower resolutions, because a larger chunk of image with more complete context information is inserted at every point. Hence, a few percentage points lower insertion metric at higher resolutions do not necessarily mean the heatmaps are any worse. In practice, $28\times 28$ heatmaps are usually more visually appealing, but in order to capture thin parts, we sometimes need to resort to $224\times 224$ (Fig.~\ref{fig:nonoise}).
 \newsavebox{\runtime}
 \begin{lrbox}{\runtime}
\begin{tabular}{||@{  }l@{  }|@{  }c@{  }|@{  }c@{  }|@{  }c@{  }|@{  }c@{  }|@{  }c@{  }||}
\hline
  Running time (s)           & 224$\times$224         & 112$\times$112        & 28$\times$28          & 14$\times$14     &7$\times$7    \\ \hline\hline
Mask   & 17.03         & 14.61        & 14.66        & 14.35   &  14.24   \\ \hline
GradCam         & -- --         & -- --            & -- --            & -- --   &    ${\bm <}$\textbf{1}      \\ \hline
RISE         & 61.77         & -- --            & -- --            & -- --   &    -- --    \\ \hline
Integrated Gradients  & ${\bm <}$\textbf{1}         & -- --            & -- --            & -- --   &    -- --   \\ \hline
I-GOS (ours)     & \text{6.07} & \textbf{5.73} & \textbf{5.70} & \textbf{5.63}& \text{5.62}\\ \hline
\end{tabular}
 \end{lrbox}

 \begin{table}[t]   
 \caption{ Comparative evaluation in terms of runtime (averaged on $5,000$ images) on the ImageNet dataset using ResNet50 as the base model.} 
 \centering   
       \scalebox{0.80}{\usebox{\runtime}}
 \label{tab:time}
 \end{table}


{\bf Speed:} In Table \ref{tab:time}, we summarize the average runtime for Mask, RISE, GradCam, Integrated Gradients, and I-GOS on the ImageNet dataset using ResNet50 as the base model.
For each approach, we only use one Nvidia 1080Ti GPU. For I-GOS, the maximal iteration is $15$; for Mask, the maximal iteration is $500$.
Our approach is faster than Mask and RISE. Especially, it converges quickly, with the average number of iterations to converge being $13$ and the time for each iteration being $0.38s$.
The average running times for the backpropagation-based methods are all less than $1$ second. However, our approach achieve much better performance than these approaches, especially with higher resolutions.
To the best of our knowledge, our approach I-GOS is the fastest among the perturbation-based methods, as well as the one with the best performance in deletion and insertion metrics among all heatmap approaches.

\begin{figure}[t]
\centering
\includegraphics[width=1\columnwidth]{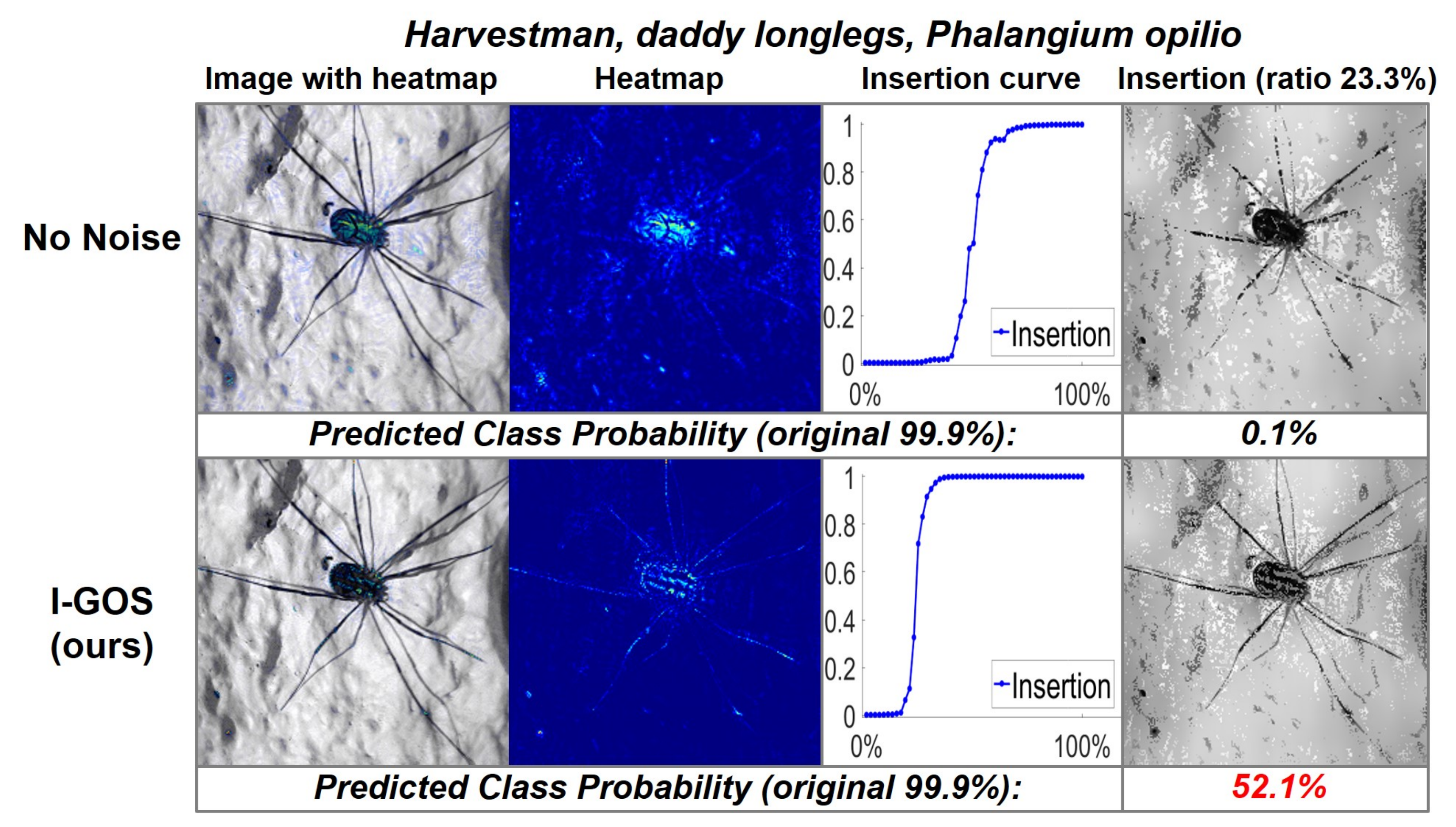}
\vskip -0.1in
\caption{Examples from ablation studies (at $224 \times 224$ resolution). With added noise, the heatmap successfully reveals the entire legs of \textit{daddy longlegs}, leading to better insertion metric, whereas without noise it is more adversarial (maybe merely by breaking each leg, CNN confidence is already reduced), leading to worse insertion metric} 
\label{nonoise}
\end{figure}

 \newsavebox{\ablation}
 \begin{lrbox}{\ablation}
\begin{tabular}{||l|l|l|l|l||}
\hline
                        & \multicolumn{2}{l|}{224$\times$224}      & \multicolumn{2}{l|}{28$\times$28}        \\ \hline\hline
     I-GOS                   & Deletion        & Insertion       & Deletion        & Insertion       \\ \hline\hline
Ours            & 0.0336          & \textbf{0.5246} & 0.0899          & \textbf{0.5701} \\ \hline
No TV term      & \textbf{0.0308} & 0.3712          & \textbf{0.0841} & 0.5181          \\ \hline
No noise        & 0.0559          & 0.4194          & 0.0872          & 0.5634          \\ \hline
Fixed step size & 0.0393          & 0.5024          & 0.0906          & 0.5403          \\ \hline
\end{tabular}
 \end{lrbox}

 \begin{table}[]   
 \caption{The results of the ablation study on VGG19.} 
 \centering   
       \scalebox{0.80}{\usebox{\ablation}}
 \label{tab:abl}
 \end{table}


{\bf Ablation Studies:} We show the results of ablation studies in Table \ref{tab:abl}. 
From Table \ref{tab:abl} we observe that without the TV term, insertion scores would indeed suffer significantly while deletion scores do not change much, indicating that the TV term is important to avoid adversarial masks.
The random noise introduced in section {\em Avoiding adversarial examples} of the paper is very useful when the resolution of the mask is high (e.g, 224$\times$224). From Fig. \ref{nonoise} we observe that I-GOS with noise can achieve much better insertion scores than without noise for the same insertion ratio. When the resolution is low (e.g, 28$\times$28), the noise is not that important since low resolution can already avoid adversarial examples. When we utilize a fixed step size (the step size is $1$ in Table \ref{tab:abl}), both deletion and insertion scores become worse, showing the utility of the line search.

{\bf Failure Case:} Fig. \ref{badexample} shows one failure case, where I-GOS found an adversarial mask and the insertion score did not increase till the end. Our observation is that optimization-based methods such as I-GOS usually do not work well when the deep model's prediction confidence is very low (less than $0.01$), or when the deep model makes a wrong prediction.

\begin{figure}[t]
\centering
\includegraphics[width=.95\columnwidth]{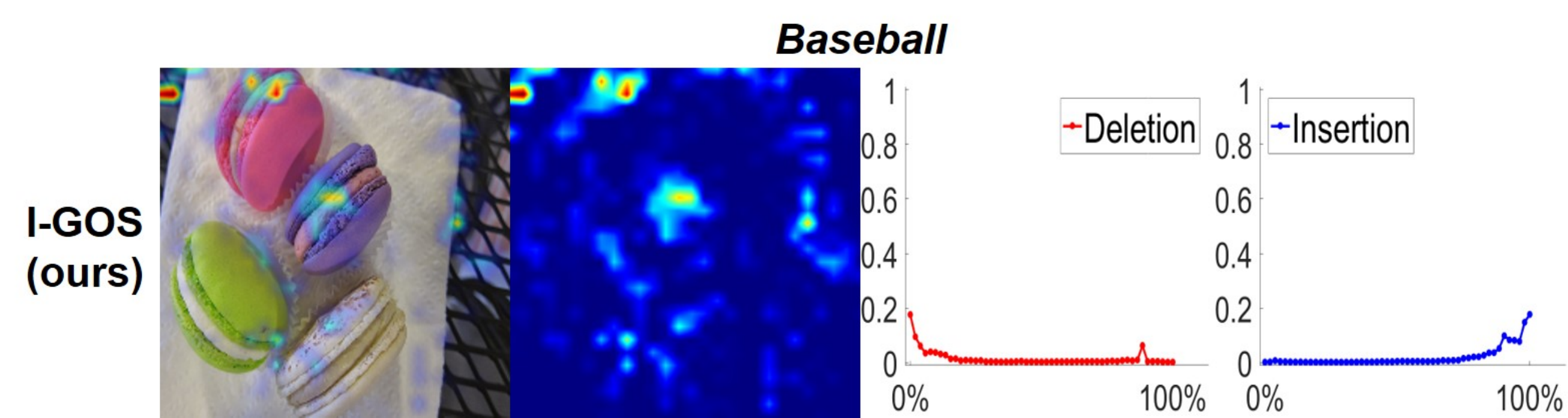}
\caption{One failure case for I-GOS, insertion curve does not move until almost all pixels have been inserted.} 
\label{badexample}
\end{figure}

 \newsavebox{\loss}
 \begin{lrbox}{\loss}
\begin{tabular}{||l||l|l||l|l||l|l||}
\hline
                                                              & \multicolumn{2}{c||}{$\lambda_1 = 0.01, \lambda_2 = 0.2$} & \multicolumn{2}{c||}{$\lambda_1 = 0.1, \lambda_2 = 2$} & \multicolumn{2}{c||}{$\lambda_1 = 1, \lambda_2 = 20$} \\ \hline
                                                              & I-GOS                       & Mask                       & I-GOS                     & Mask                      & I-GOS                     & Mask                     \\ \hline\hline
\begin{tabular}[c]{@{}l@{}}Total\\ loss\end{tabular}          & 0.2241                      & 0.3349                     & 0.3739                    & 0.4857                    & 0.6098                    & 0.6794                   \\ \hline
Deletion                                                      & 0.0825                      & 0.1056                     & 0.0861                    & 0.1178                    & 0.0899                    & 0.1340                    \\ \hline
Insertion                                                     & 0.5418                      & 0.5335                     & 0.5624                    & 0.5307                    & 0.5701                    & 0.5207                   \\ \hline
\end{tabular}
 \end{lrbox}

 \begin{table}[t]
 \caption{The optimization loss on VGG19 for resolution $28\times 28$.} 
 \centering   
       \scalebox{0.78}{\usebox{\loss}}
 \label{tab:loss}
 \end{table}

{\bf Convergence:} For the values of the objective after convergence with Mask \cite{ClassicMask} vs. the proposed I-GOS, 
Table \ref{tab:loss} shows the comparison at $28\times 28$ with different parameters. Best parameters were used for each approach in Table \ref{tab:insertionvgg} (0.01/0.2 for Mask and 1/20 for I-GOS). It can be seen at every parameter setting I-GOS has lower total loss than Mask (total loss is higher with larger $\lambda_1$ and $\lambda_2$ since the L1+TV terms have higher weights in total loss).

\section{Conclusion}
In this paper, we propose a novel visualization approach I-GOS, which utilizes integrated gradients to optimize for a heatmap. We show that the integrated gradients provides a better direction than the gradient when a good baseline is known for part of the objective of the optimization. The heatmaps generated by the proposed approach are human-understandable and more correlated to the decision-making of the model. 
Extensive experiments are conducted on three benchmark datasets with four pretrained deep neural networks, which shows that I-GOS
 advances state-of-the-art deletion and insertion scores on all heatmap resolutions.


\section{Acknowledgments}
This work was partially supported by DARPA contract N66001-17-2-4030.


\section*{Supplementary Material}


\subsection*{{\uppercase\expandafter{\romannumeral1}. Properties of the Integrated Gradient in Quadratic Functions}}

\begin{proposition} The integrated gradients reduce to a scaling on the conventional gradient in a quadratic function if the baseline is the optimum.
\end{proposition}
\begin{proof}

Given a quadratic function $f({\mathbf x})= {\mathbf x}^TA{\mathbf x} +b^T{\mathbf x} +c$, we have its conventional gradient as:
$\nabla f({\mathbf x}) = (A+A^T){\mathbf x} +b$.
Considering a straight-line path from the current point ${\mathbf x}_k$ to the baseline ${\mathbf x}_0$, for point ${\mathbf x}_s$ along the path, we have:
${\mathbf x}_s = {\mathbf x}_0+\frac{s}{S}({\mathbf x}_k-{\mathbf x}_0)$,
{
\begin{align}
\nabla f({\mathbf x}_s) &= (A+A^T){\mathbf x}_s +b \notag\\
&= (A+A^T)\left( {\mathbf x}_0+\frac{s}{S}({\mathbf x}_k-{\mathbf x}_0)\right)+b \notag\\
&= \frac{s}{S}(A+A^T){\mathbf x}_k +\frac{S-s}{S}(A+A^T){\mathbf x}_0 +b \notag\\
&= \frac{s}{S}\nabla f({\mathbf x}_k) + \frac{S-s}{S} \nabla f({\mathbf x}_0),
\label{eq:Pxs}
\end{align}
}Thus, we obtain the integrated gradient along the straight-line path as:
{
\begin{align}
\nabla^{IG} f({\mathbf x}_k) &=\frac{1}{S} \sum_{s=1}^{S} \nabla f({\mathbf x}_s) \notag\\
&=\frac{S+1}{2S}\nabla f({\mathbf x}_k) + \frac{S-1}{2S} \nabla f({\mathbf x}_0),
\label{eq:IG0}
\end{align}
}When the baseline ${\mathbf x}_0$ is the optimum of the quadratic function, $\nabla f({\mathbf x}_0) = 0$, and then
{
\begin{align}
\nabla^{IG} f({\mathbf x}_k) = \frac{S+1}{2S}\nabla f({\mathbf x}_k).
\label{eq:IG1}
\end{align}
}Hence, the integrated gradients reduce to a scaling on the conventional gradient.

In this case, the revised Armijo condition also reduces to the conventional Armijo condition up to a constant.

\end{proof}

\subsection*{{\uppercase\expandafter{\romannumeral2}. Pointing Game}}

For the pointing game task, following \cite{2018RISE}, if the most salient pixel lies inside the human annotated bounding box of an object, it is counted as a hit.
The pointing game accuracy equals to $\frac{\# Hits}{\# Hits+\# Misses}$, averaged over all categories.
We utilize two pretrained VGG16 models from \cite{2018RISE} to test $2,000$ randomly selected images from the validation set of MSCOCO, and $2,000$ randomly selected images from the test set of VOC07, respectively. 
Table \ref{tab:pg} shows the comparative evaluations of I-GOS with other state-of-the-art approaches in terms of mean accuracy in the pointing game on MSCOCO and PASCAL, respectively.
Here we utilize the same pretrained models from \cite{2018RISE}.
Hence, we list the pointing game accuracies reported in the paper except for Mask and our approach I-GOS.
From Table \ref{tab:pg} we observe that, I-GOS beats all the other compared approaches except for RISE, and it improves significantly over of the Mask.
During the experiments we notice that, some object labels for MSCOCO and PASCAL in the pointing game have very small output scores for the pretrained VGG16 models, which affects the optimization greatly for both Mask and I-GOS. 
RISE does not seem to suffer from this. We believe RISE may be good at the pointing game, but its randomness would generally lead to a mask that is too diffuse, which significantly hurts its deletion and insertion scores (Table \ref{tab:insertionvgg} and Table \ref{tab:insertionres} in the paper), 
while our approach generates a much more concise heatmap.

 \newsavebox{\point}
 \begin{lrbox}{\point}
\begin{tabular}{||l|l|l||}
\hline
Mean Acc (\%)          &    MSCOCO & VOC07\\ \hline\hline
AM \cite{SimonyanVZ13}           & 37.10   &      76.00      \\ \hline
Deconv  \cite{MatDeconv}       & 38.60   &     75.50       \\ \hline
MWP   \cite{zhang16excitationBP}         & 39.50 &   76.90           \\ \hline
Excitation BP \cite{zhang16excitationBP} & 49.60 &   80.00           \\ \hline
RISE   \cite{2018RISE}        & \textbf{50.71}  &   {\bf 87.33}         \\ \hline\hline
Mask \cite{ClassicMask} (14$\times$14)  & 40.03  &    79.45        \\ \hline
Mask \cite{ClassicMask} (28$\times$28)  & 43.24  &    77.57        \\ \hline\hline
I-GOS (ours) (14$\times$14) & 47.16   &     {\bf 85.81}      \\ \hline
I-GOS (ours) (28$\times$28) & \bf{49.62}  &    83.61        \\ \hline
\end{tabular}
 \end{lrbox}

 \begin{table}[b]   
 \caption{\small Mean accuracy (\%) in the pointing game for VGG16 on MSCOCO and PASCAL VOC07, respectively.} 
 \centering   
       \scalebox{0.80}{\usebox{\point}}
 \label{tab:pg}
 \vskip -0.15in
 \end{table}

 \subsection*{{\uppercase\expandafter{\romannumeral4}. Adversarial Examples}}
Figure \ref{fig:ad}-\ref{fig:ad2} shows some examples when using I-GOS to visualize adversarial examples. 
 Here we utilize the MI-FGSM method \cite{Dong_2018_CVPR} on VGG19 to generate adversarial examples. 
 From Fig.~\ref{fig:ad}-\ref{fig:ad2} we observe that the heatmaps for the original images and for the adversarial examples generated by I-GOS are totally different.
 For the original image, I-GOS can often lead to a high classification confidence on the original class by inserting a small portion of the pixels.
 For the adversarial image though, almost the entire image needs to be inserted for CNN to predict the adversarial category. We note that we are not presenting I-GOS as a defense against adversarial attacks, and that specific attacks may be designed targeting the salient regions in the image. However, these figures show that the I-GOS heatmap and the insertion metric are robust against those full-image based attacks and not performing mere image reconstruction.

 \begin{figure*}[]   
 \centering   
\includegraphics[width=2\columnwidth]{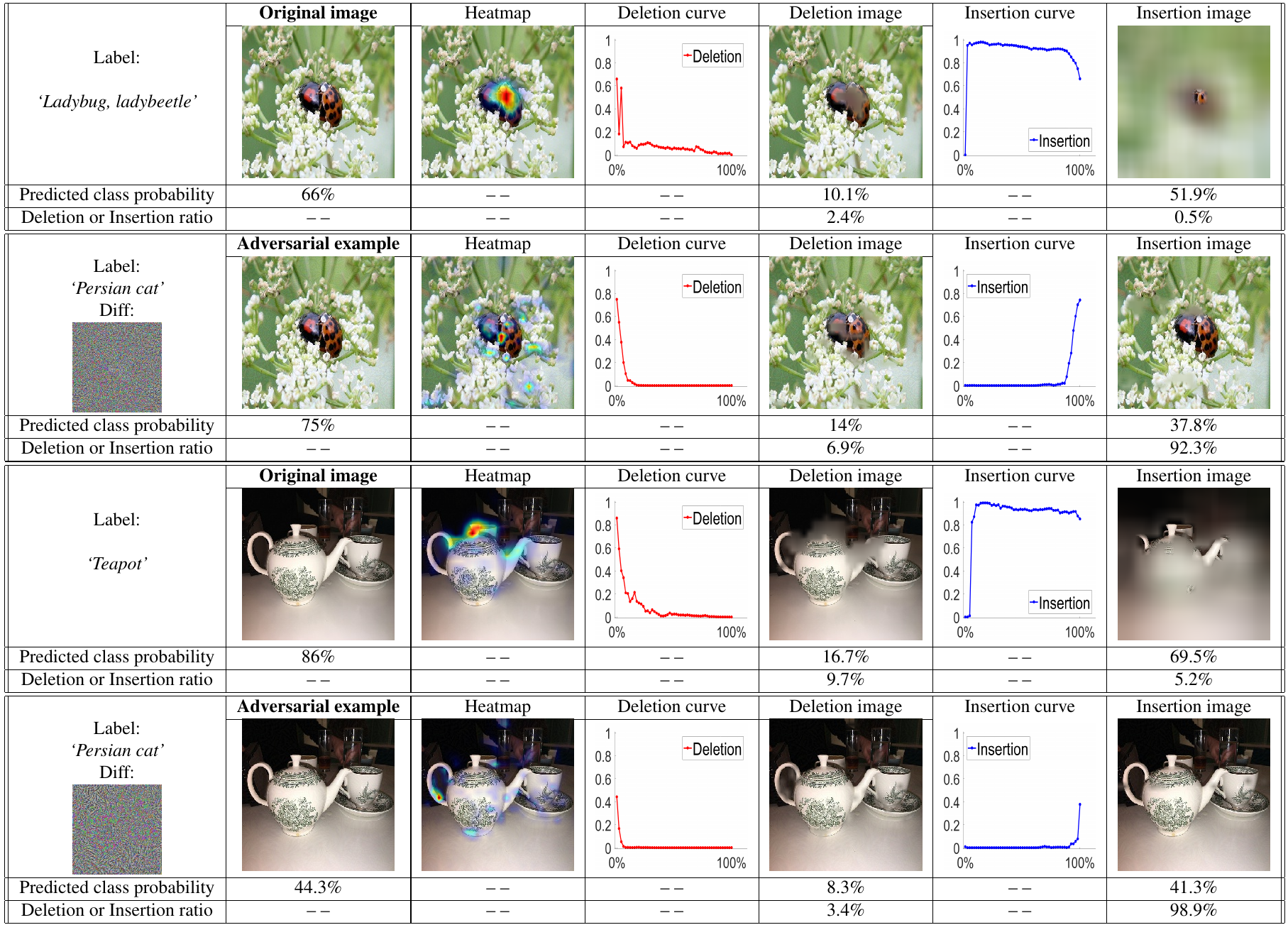}
\caption{\small The top row are original images and their heatmaps generated by I-GOS; the bottom row are adversarial examples and their heatmaps generated by I-GOS. 
The red plot illustrates how the CNN predicted probability drops with more areas masked, and the blue plot illustrates how the prediction increases with more areas revealed. The x axis for the red/blue plot represents the percentage of pixels masked/revealed; the y axis for the red/blue plot represents the predicted class probability.
One can see on normal images, CNN can classify with only the highlighted parts revealed, whereas on adversarial images one would almost need to insert the entire image to make the CNN to classify to the adversarial label.} 
 \label{fig:ad}
 \vskip -0.15in
 \end{figure*}

 \begin{figure*}[]   
 \centering   
\includegraphics[width=2\columnwidth]{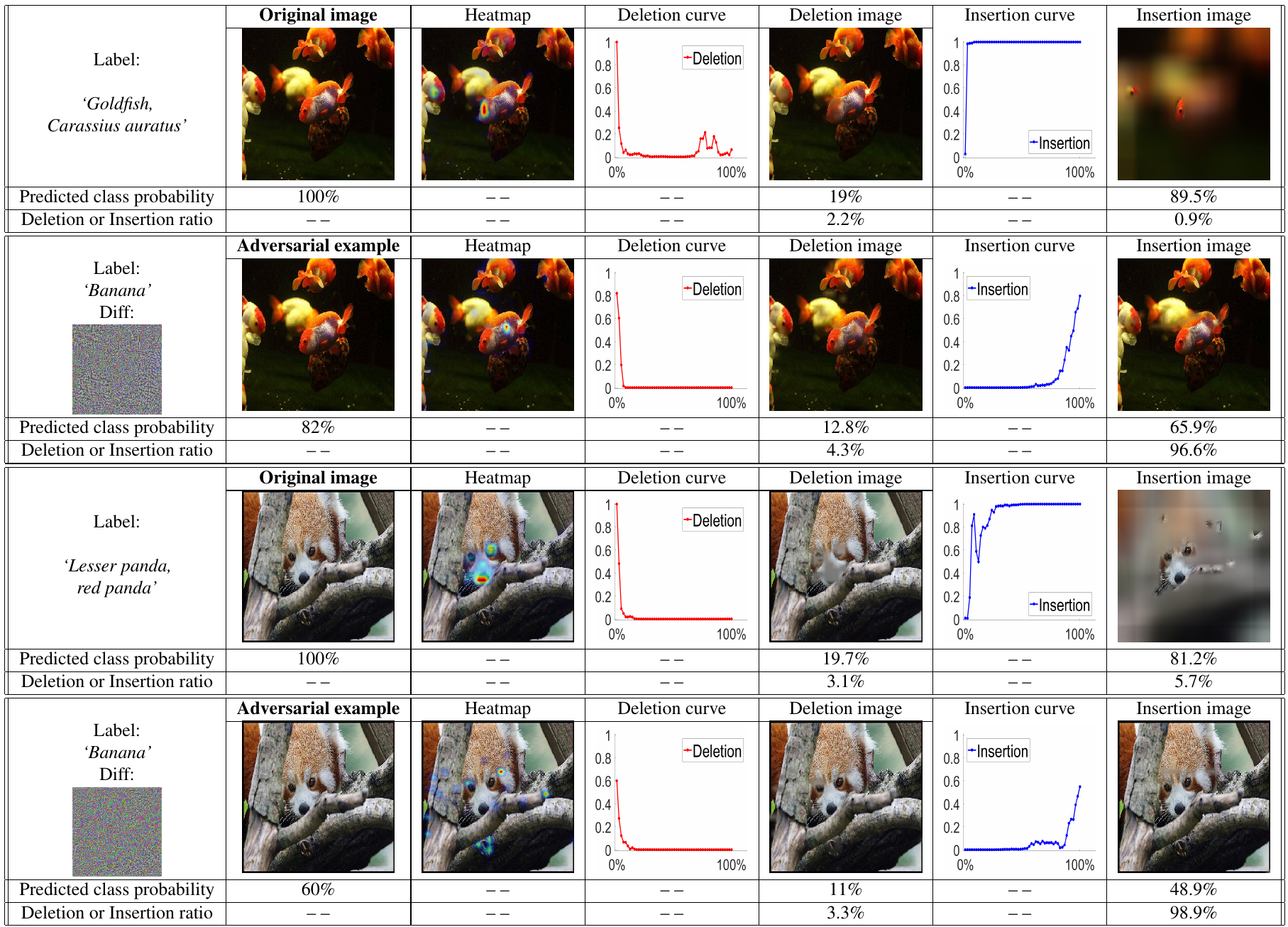}
\caption{\small The top row are original images and their heatmaps generated by I-GOS; the bottom row are adversarial examples and their heatmaps generated by I-GOS, see Fig. \ref{fig:ad} caption for explanations of the meaning of the curves. One can see on normal images, CNN can classify with only the highlighted parts revealed, whereas on adversarial images one would almost need to insert the entire image to make the CNN to classify to the adversarial label.} 
 \label{fig:ad2}
 \vskip -0.15in
 \end{figure*}

\begin{figure*}[t]
\vskip -0.05in
\centering
\includegraphics[width=1.9\columnwidth]{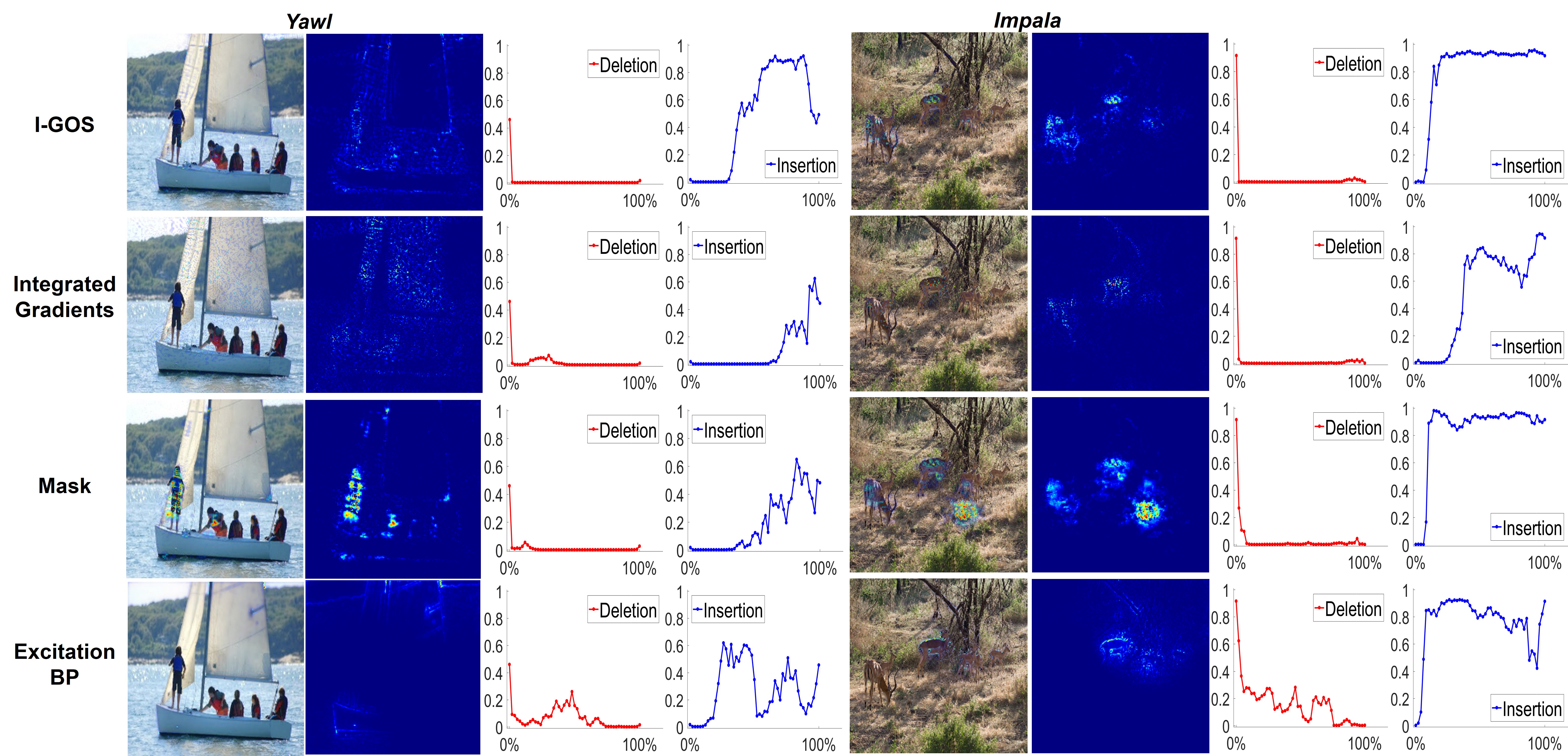}
\caption{A comparison among different approaches with heatmaps of $224\times 224$ resolution. The red plot illustrates how the CNN predicted probability drops with more areas masked, and the blue plot illustrates how the prediction increases with more areas revealed. 
The x axis for the red/blue plot represents the percentage of pixels masked/revealed;
the y axis for the red/blue plot represents the predicted class probability.
One can see with I-GOS the red curve drops earlier and the blue plot increases earlier, leading to more area under the insertion curve (insertion metric) and less area under the deletion curve (deletion metric). (Best viewed in color) } 
\label{fig:compare1b}
\vskip -0.15in
\end{figure*}

\begin{figure*}[t]
\centering
\includegraphics[width=1.9\columnwidth]{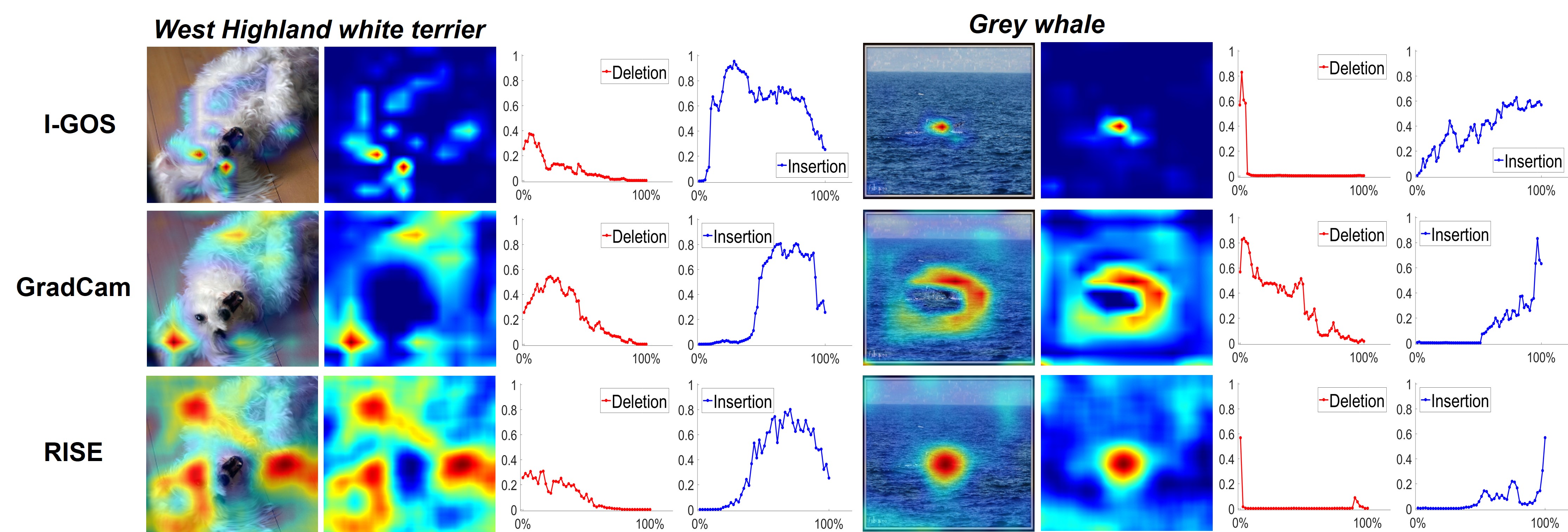}
\caption{Comparisons between GradCam, RISE, and I-GOS, see Fig. \ref{fig:compare1b} caption for explanations of the meaning of the curves.}
\label{fig:GradCamRISEb}
\vskip -0.15in
\end{figure*}


\begin{figure*}
   \centering
\includegraphics[width=1.7\columnwidth]{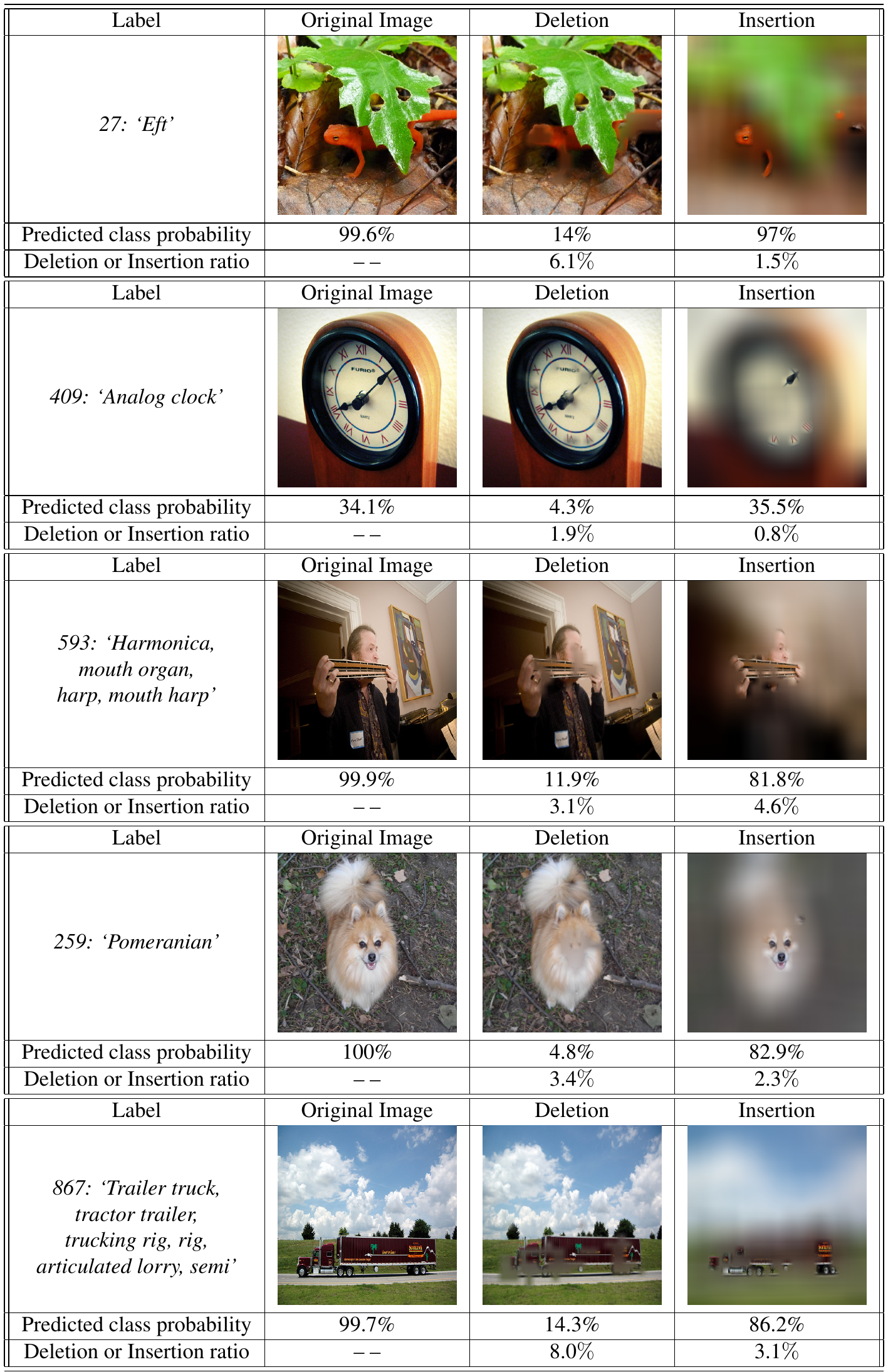}

    \caption{Examples generated by I-GOS in the deletion and insertion task using VGG19 as the baseline model.}
    \label{fig:vgg1} 
\end{figure*}

\begin{figure*}
   \centering
\includegraphics[width=1.7\columnwidth]{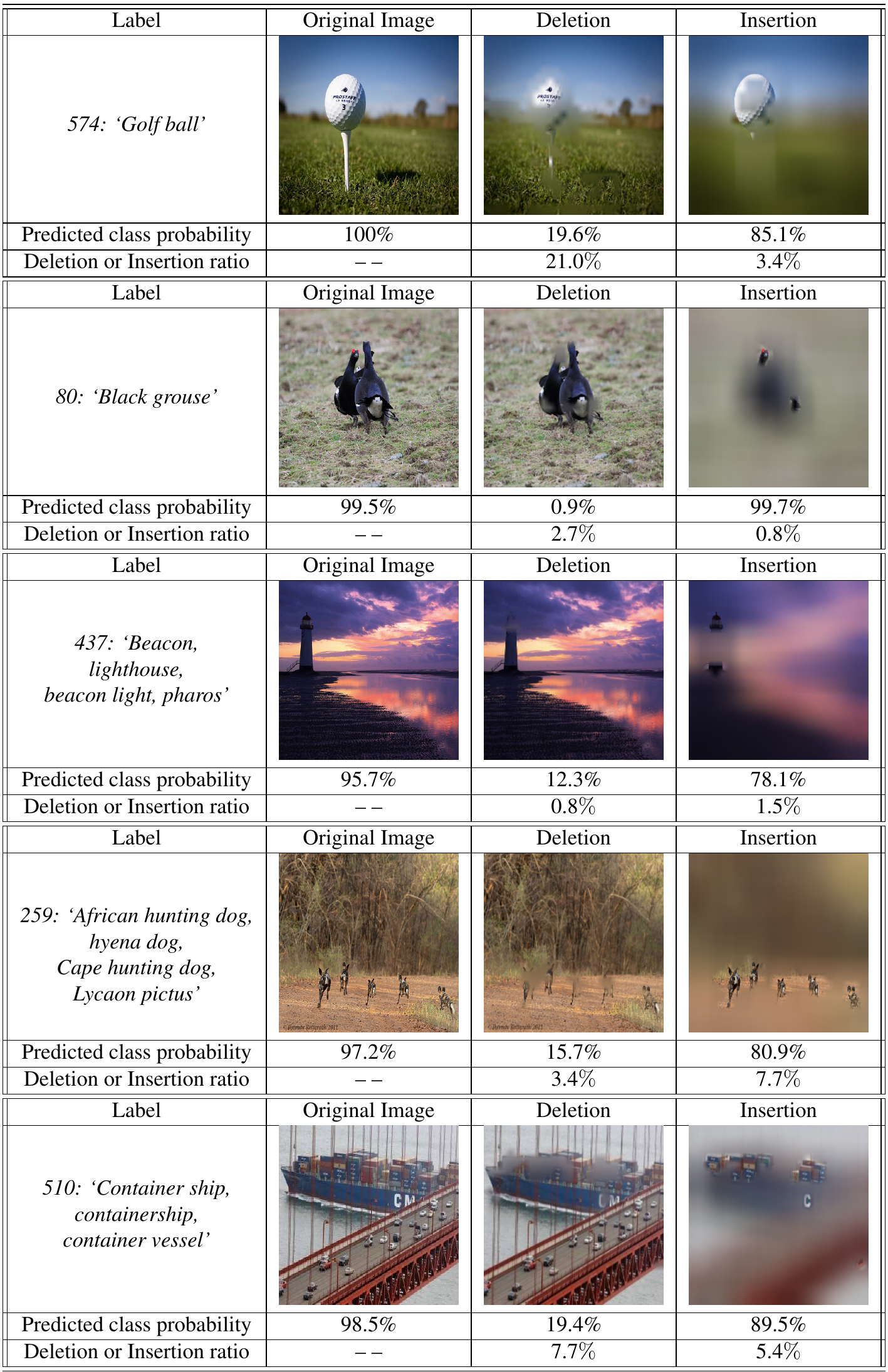}

    \caption{Examples generated by I-GOS in the deletion and insertion task using VGG19 as the baseline model.}
    \label{fig:vgg2} 
\end{figure*}

\begin{figure*}
   \centering
\includegraphics[width=1.7\columnwidth]{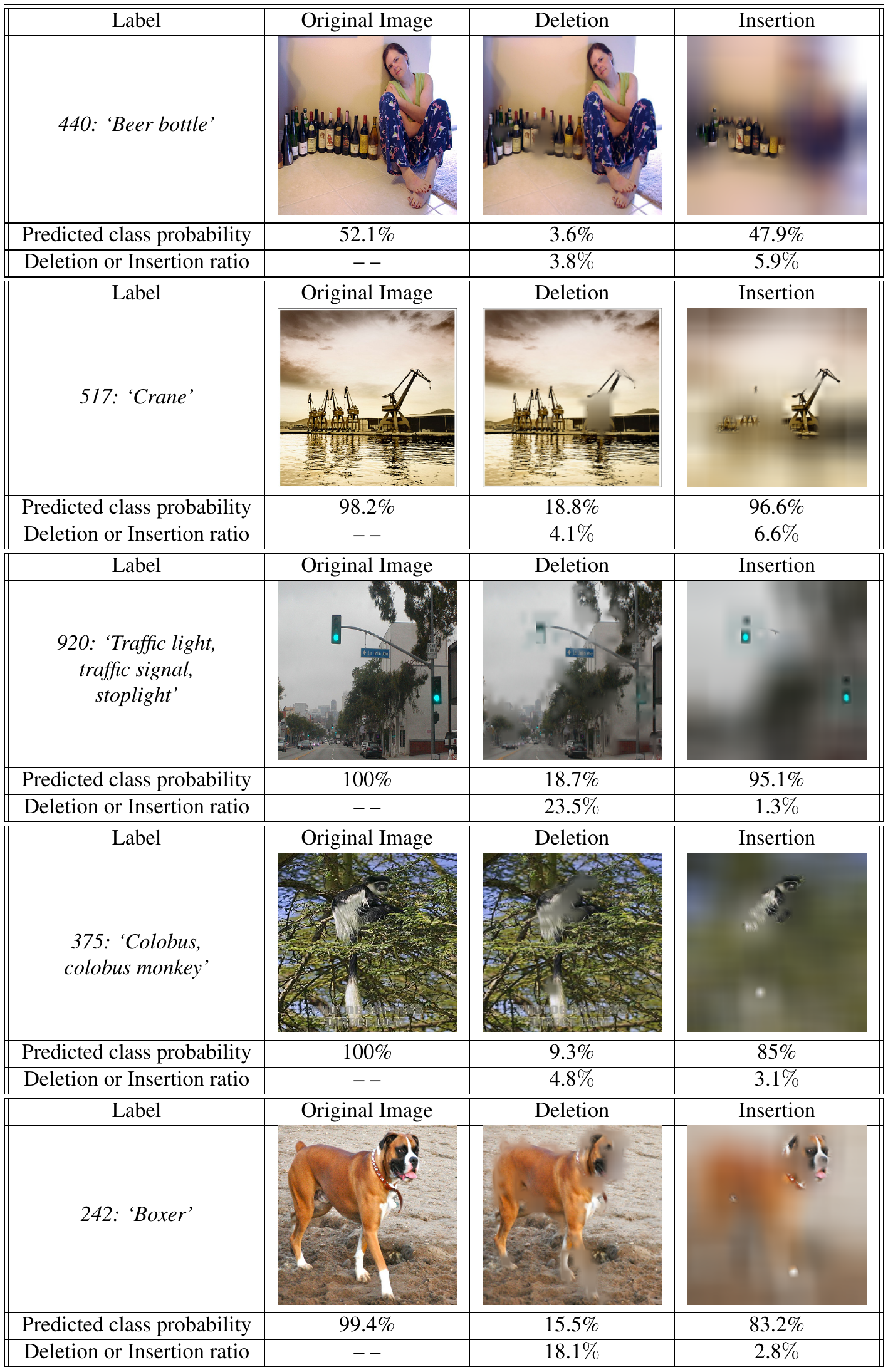}

    \caption{Examples generated by I-GOS in the deletion and insertion task using Resnet50 as the baseline model.}
    \label{fig:resnet1} 
\end{figure*}

\begin{figure*}
   \centering
\includegraphics[width=1.7\columnwidth]{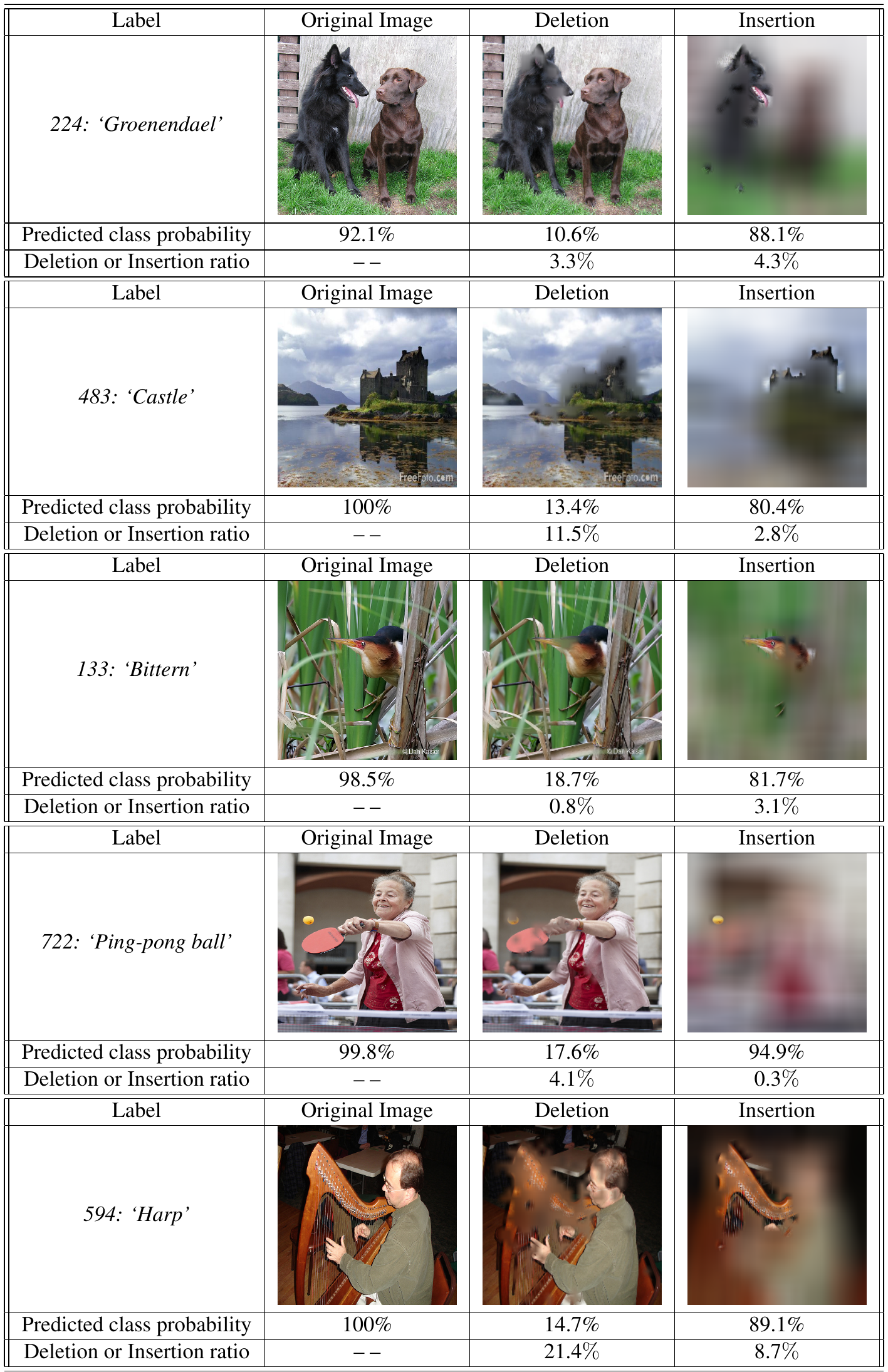}

    \caption{Examples generated by I-GOS in the deletion and insertion task using Resnet50 as the baseline model.}
    \label{fig:resnet2} 
\end{figure*}


\subsection*{{\uppercase\expandafter{\romannumeral5}. Deletion and Insertion Visualizations}}
Fig. \ref{fig:compare1b} shows more comparison examples between different approaches on $224 \times 224$ heatmaps.
Fig. \ref{fig:GradCamRISEb} shows more visual comparisons between our approach, GradCAM, and RISE.
From Fig. \ref{fig:compare1b} we can see that, for Mask, it focuses on person instead of {\em Yawl} on the left image, and focuses on grass instead of {\em Impala} on the right image, indicating that sometimes the optimization can be stuck in a bad local optimum.
From Fig. \ref{fig:GradCamRISEb} we observe that sometimes GradCAM also fires on image border, corner, or irrelevant parts of the image ({\em Grey whale} in Fig. \ref{fig:GradCamRISEb}), which results in bad deletion and insertion scores. 
And the randomness on the mask indeed limits the performance of RISE ({\em West Highland white terrier} in Fig. \ref{fig:GradCamRISEb}).


Fig. \ref{fig:vgg1}-\ref{fig:vgg2} show some examples generated by our approach I-GOS in the deletion and insertion task using VGG19 as the baseline model.
Fig. \ref{fig:resnet1}-\ref{fig:resnet2} show some examples generated by I-GOS in the deletion and insertion task using Resnet50 as the baseline model.
The deletion or insertion image is generated by $I_0 \odot {\text up}(M) + \tilde{I}_0 \odot \left({\bf 1}-{\text up}(M)\right)$, where the resolution of $M$ is $28\times 28$.
For deletion image, we initialize the mask $M$ as matrix of ones, then set the top $N$ pixels in the mask to $0$ based on the values of the heatmap, where the deletion ratio represents the proportion of pixels that are set to $0$.
For insertion image, we initialize mask $M$ as matrix of zeros, then set the top $N$ pixels in the mask to $1$ based on the values of the heatmap, where the insertion ratio represents the proportion of pixels that are set to $1$.
In Fig. \ref{fig:vgg1}-\ref{fig:resnet2}, the masked/revealed regions of the images may seem a little larger than the number of the deletion/insertion ratios. The reason is that after upsampling the mask $M$, some pixels on the border may have values between $0$ and $1$, resulting in larger regions to be masked or revealed.
The predicted class probability is the output value after softmax for the same category using the original image, the deletion image, and the insertion image as the input, respectively.
From Fig. \ref{fig:vgg1}-\ref{fig:resnet2} we observe that the proposed approach I-GOS can utilize a low deletion ratio to achieve a low predicted class probability for the deletion task, and a low insertion ratio to achieve a high predicted class probability for the insertion task at the same time, indicating that I-GOS truly discovers the key features of the images that the CNN network is using.
Especially, we realize that CNN is indeed fixating on very small regions in the image and very local features in many cases to make a prediction, e.g. in \textit{Pomeranian}, the face of the dog is utmostly important. Without the face the prediction is reduced to almost zero, and with only the face and a rough outline of the dog, the prediction is almost perfect.  The same can be said for \textit{Eft}, \textit{Black grouse}, \textit{lighthouse} and \textit{boxer}. Interestingly, for \textit{Container ship} and \textit{trailer truck}, their functional parts are extremely important to the classification. \textit{Trailer truck} almost cannot be classified without the wheels (and could be classified with only the wheels), and \textit{container ship} cannot be classified without the containers (and could be classified with almost only the containers and a rough outline of the ship).



{\small
\bibliographystyle{aaai}
\bibliography{egbib}
}

\end{document}